\documentclass[journal]{IEEEtran}
\usepackage{cite}
\usepackage{comment}
\usepackage{soul}
\usepackage{hyperref}
\hypersetup{
    colorlinks,
    linkcolor={red!50!black},
    citecolor={blue!50!black},
    urlcolor={blue!80!black}
}

\ifCLASSINFOpdf
\else
\fi

\usepackage{amsthm,amssymb,amsmath,optidef,subfigure,xcolor}
\usepackage[printonlyused]{acronym} 
\usepackage{cleveref}
\usepackage{ulem}
\crefname{figure}{Fig.}{Figs.}
\crefname{table}{Tab.}{Tabs.}
\crefname{equation}{Eq.}{Eqs.}

\usepackage{tabu}
\usepackage{float}

\newtheorem{prop}{Proposition}
\newtheorem{corollary}{Corollary}
\newtheorem{lemma}{Lemma}

\def\R{\mathbb{R}}

\def\x{\mathbf{x}}
\def\z{\mathbf{z}}
\def\Z{\mathbf{Z}}
\def\u{\mathbf{u}}
\def\U{\mathbf{U}}
\def\S{\mathbf{S}}
\def\V{\mathbf{V}}
\def\N{\mathbf{N}}
\def\L{\mathcal{L}}

\def\threLi{LI}

\def\c{\mathbf{c}}
\def\a{\mathbf{a}}
\def\H{\mathbf{H}}
\def\S{\mathbf{S}}
\def\J{\mathcal{J}}
\def\r{r}
\def\dimx{n}
\def\dimz{m}
\def\time{t}
\def\noise{\mathbf{e}}
\def\h{\boldsymbol{h}}
\def\ex{\hat{\mathbf{x}}}
\def\ez{\hat{\mathbf{z}}}
\def\eR{\mathbf{R}}

\newcommand{\jc}[1]{\textbf{\color{blue}[[JC: #1]]}}

\newcommand{\rev}[1]{{\color{blue}#1}}

\begin{document}

\title{
Physically Consistent Null Space Alignment for Detection of Low-Magnitude False Data Injection Attacks
}

\author{
}

\IEEEaftertitletext{}
\author{\IEEEauthorblockN{
\textbf{Xin Li}\IEEEauthorrefmark{1},
\textbf{Chenhan Xiao}\IEEEauthorrefmark{2}, 
\textbf{Jonathan Cohen}\IEEEauthorrefmark{1}, 
\textbf{Aviad Elyashar}\IEEEauthorrefmark{3}\IEEEauthorrefmark{4}, 
\textbf{Yang Weng}\IEEEauthorrefmark{2},
\textbf{Rami Puzis}\IEEEauthorrefmark{1}\IEEEauthorrefmark{4}
}

\IEEEauthorblockA{\IEEEauthorrefmark{1}}Faculty of Computer and Information Science, Ben-Gurion-University, Be'er Sheva, Israel \\
\IEEEauthorblockA{\IEEEauthorrefmark{2}School of Electrical, Computer and Energy Engineering, Arizona State University, AZ, USA}\\
\IEEEauthorblockA{\IEEEauthorrefmark{3}Department of Computer Science, Shamoon College of Engineering, Be'er Sheva, Israel}\\
\IEEEauthorblockA{\IEEEauthorrefmark{4}Cyber@BGU, Ben-Gurion University of Negev, Be'er Sheva, Israel}\\
}

\renewcommand{\rev}[1]{#1}

\maketitle

\begin{abstract}
\Acfp{fdia} introducing small measurement perturbations can still cause large deviations in power system state estimation when the injected signals align with the pseudo-null space of the system model. 
Existing model- and data-driven detectors may fail to identify such low-magnitude but high-impact attacks because residual tests ignore changes hidden in the pseudo-null space, while subspace learning methods capture correlation patterns without enforcing physical consistency.
%
\rev{
This paper proposes \ac{ping}, a framework that detects stealthy \acp{fdia} by preserving, through preprocessing, the geometric correspondence between the physical null space and the measurement-derived pseudo-null space.
The key point is a \ac{pscp} step that re-expresses measurements in the physical coordinate frame before subspace extraction. 
We prove that \ac{pscp} preserves the separation between row space and its orthogonal complement, a property that conventional per-feature standardization violates.
This keeps the \ac{svd}-derived pseudo-null subspace aligned with the physical residual space without explicit knowledge of $\H$.
}
%
\rev{Experiments on IEEE 14-, 30-, 57-, and 118-bus systems confirm this principle in practice: stealthy attacks that evade XTM, LSTM, AE and Isolation Forest baselines appear as clear deviations in the aligned subspace, yielding higher F1-score and detection accuracy while remaining robust under partial observability and realistic PMU noise.}

\end{abstract}


\acresetall
\section{Introduction}
\label{sec:introduction}

The increasing digitalization of power grids has enhanced operational efficiency and situational awareness, but has also expanded the cyber-attack surface of critical infrastructures. 
Past incidents such as the 2015 Ukraine blackout \cite{lightsout}, the 2018 intrusion into U.S. grid networks \cite{cnn2018us}, and the 2022 coordinated attack on U.S. energy entities \cite{cyberattacknews} illustrate how malicious data manipulation can bypass traditional protection layers.
While most existing defenses can handle medium and large intrusions, recent studies and operational reports show that even small, well-coordinated measurement perturbations can destabilize automatic generation control and voltage regulation, leading to wide-area oscillations or hidden economic losses \cite{hong2021data, paul2022modified, mohamed2021false, rahman2014impact}.
These findings have shifted attention from large cyberattacks to stealthy low-magnitude injections that exploit the physics of state estimation to produce large state deviations without visible measurement anomalies \cite{Jafari2021SmallSignalFDIA,Zhang2023FDIAWLS}.

During the past years, \ac{fdia} detection has evolved from classical model-based residual testing to advanced data-driven frameworks.
Traditional Chi-square \acp{bdd} rely on accurate knowledge of the Jacobian matrix and assume stationary measurement noise \cite{abur2004power}, while early data-driven schemes such as \ac{pca}, \ac{knn}, and Isolation Forest detect statistical outliers without explicit models~\cite{MOHAMMADPOURFARD2020105947}.
More recent approaches integrate temporal and spatial correlations. For example, \ac{lstm}-based detectors use forecast–residual mismatches to expose temporal inconsistencies \cite{Yang2020LSTMFDIA, deepSignalAnomalyDetector}, and \ac{gnn} exploit electrical topology for spatial detection \cite{Boyaci2023GNNFDIA, Takiddin2023GeneralizedGNN}.

Although both residual- and learning-based detectors perform well against medium to large, disruptive attacks, they fail to identify small, physically consistent perturbations. 
For example, model-based tests ignore variations within the pseudo-null space of the Jacobian \cite{liu2011false}, where stealthy attacks leave residuals unchanged.
Conversely, data-driven subspace and deep-learning methods, such as \ac{pca}, \acs{svd}\acused{svd}, and \ac{lstm} detectors \cite{kim2014subspace, yu2015blind, anwar2017modeling, rahman2019finding, zu2024self}, treat 
coordinated null space perturbations as noise.
The problem worsens with the growing use of \ac{ai} in the grid for forecasting and dispatch. 
As a result, \ac{ai}-based algorithms may reshape network operating points \cite{shi2020artificial}, causing $\H$ to drift and allowing small null space attacks to mimic legitimate fluctuations while inducing large state estimation errors.


\rev{This challenge motivates a different design philosophy. 
Instead of detecting anomalies within model-defined or 
data-driven subspaces in isolation, we ensure that the preprocessing step preserves the geometric correspondence 
between them, so that the data-driven subspace remains a 
meaningful proxy of the physical null space.
Under normal conditions, measurements lie within $\mathcal{H}_{\mathrm{DC}}$  up to noise, where
$\mathcal{H}_{\mathrm{DC}} = \{\H\x 
: \x \in \R^{\dimx}\}$ denotes the column space of $\H$.
Preprocessing steps prior to \ac{svd} must preserve this range. Otherwise, the empirical residual subspace drifts away from $\mathcal{N}(\H^\top)$ and loses physical interpretability.
While prior research examined the Jacobian null space~\cite{liu2011false} or applied data-driven \ac{svd}/\ac{pca} projections~\cite{yang2022blind,rahman2019finding}, none designed a preprocessing pipeline that conserves this geometric correspondence. By enforcing null space conservation at the preprocessing stage, our framework detects low-magnitude attacks that are missed by both residual and learning-based methods.}

Although directly comparing the model- and data-derived null spaces may appear sufficient, it fails without proper alignment of the measurement space with the physical model.
For example, conventional normalization rescales magnitudes statistically, but not physically, often masking small attacks by removing geometric cues tied to network laws.
To overcome this challenge, we introduce a \ac{pscp} that reprojects measurements into the model’s coordinate frame before standardization.
This step preserves Kirchhoff-consistent relationships and enables precise subspace comparison, forming the foundation of the proposed \ac{ping} framework.

The proposed \ac{ping} framework is evaluated on IEEE 14-, 30-, 57-, and 118-bus systems with realistic load and \ac{pmu}\acused{pmu} noise conditions.
\ac{ping} consistently detects low-magnitude \acp{fdia} that evade Chi-square and learning-based detectors, maintaining robustness under noise and partial observability.
These results confirm that aligning physical and measurement subspaces, rather than adding complexity, is the key to detecting stealthy low-magnitude attacks.
\rev{We further give an explicit physical interpretation: \ac{pscp} preserves the network topology, branch reactances, and linearized power-balance equations, with \ac{kcl} preserved implicitly in \ac{dc} and explicitly along zero-injection chains in \ac{ac}.}


The rest of this paper is organized as follows. 
\Cref{sec:fdia_pre} reviews the \ac{fdia} problem formulation. 
\Cref{sec:approach} presents the investigated pseudo-null space \ac{fdia} and the proposed detection framework relying on \ac{ping}. 
\Cref{sec:evaluation} describes the numerical experiments and performance evaluation results, and \Cref{sec:conclusion} concludes the paper.



\section{Preliminaries}
\label{sec:fdia_pre}

For analytical clarity, we adopt the \acs{dc}\acused{dc} \ac{se} model as the foundation of our derivation, while noting that the proposed detection framework generalizes naturally to \acs{ac}\acused{ac} systems.  
The \ac{dc} formulation provides a linear structure that clearly reveals the geometric relationships among the measurement, state, and attack spaces. These relationships are central to understanding small, physically consistent \acp{fdia}.

In the \ac{dc} model, the measurement vector $\z = (z_1, \ldots, z_{\dimz}) \in \mathbb{R}^{\dimz}$ is linearly related to the state vector $\x = (x_1, \ldots, x_{\dimx}) \in \mathbb{R}^{\dimx}$, typically representing bus voltage phase angles:
\begin{equation}
    \z = \H \x + \noise,
\end{equation}
where $\H \in \mathbb{R}^{\dimz \times \dimx}$ is the Jacobian matrix determined by the network topology, and $\noise \in \mathbb{R}^{\dimz}$ is additive measurement noise.  
The noise is commonly modeled as Gaussian, $\noise \sim \mathcal{N}(\mathbf{0}, \eR)$, with diagonal covariance $\eR = \mathrm{diag}(\sigma_1^2, \ldots, \sigma_{\dimz}^2)$ that captures sensor heterogeneity.

Given a set of measurements $\z$, the system state is obtained using the \ac{wls} estimator~\cite{abur2004power}:
\begin{equation}
    \ex = \arg\min_{\x} \sum_{i=1}^{\dimz} \frac{(z_i - \H_i \x)^2}{\sigma_i^2},
    \label{eq:SE}
\end{equation}
where $\H_i$ is the $i$th row of $\H$.  
In the \ac{dc} case, this yields a closed-form solution:
\begin{equation}
    \ex = (\H^\top \eR^{-1} \H)^{-1} \H^\top \eR^{-1} \z.
\end{equation}

Under normal operating conditions, the estimated measurements $\ez = \H \ex$ differ from the observed data $\z$ only by random noise.  
Residual-based \ac{bdd} therefore tests whether the residual magnitude exceeds some threshold~\cite{abur2004power, wang2020detection}.  
The residual is
\begin{equation}
    \r = \z - \H \ex = \S \z,
\end{equation}
where $\S = \mathbf{I} - \H(\H^\top \eR^{-1}\H)^{-1}\H^\top \eR^{-1}$ is the residual sensitivity matrix satisfying $\S \H = \mathbf{0}$.  
Because $\S$ projects $\z$ onto the orthogonal complement of the column space of $\H$, any change that lies entirely within $\mathcal{H}_{\mathrm{DC}}$ (i.e., within its column space) remains invisible to the residual.  
This property forms the basis of most undetectable \ac{fdia} constructions.

In the absence of attacks, the \ac{cst} test statistic follows a chi-squared distribution:
\begin{equation}
    \J(\z) = \sum_{i=1}^{\dimz} \frac{(\S_i \z)^2}{\sigma_i^2} \sim \chi^2_{\dimz-\dimx},
\end{equation}
and data are flagged as suspicious if $\J(\z)$ exceeds a threshold $\chi^2_{(\dimz-\dimx),\,1-\alpha}$ for a given significance level $\alpha$.

\subsection{Pseudo-null space and operational interpretation.}  
A central concept for this paper is the {pseudo-null space}. In theory, state perturbations within the exact null space of $\H$ (i.e., $\c\in\mathrm{Null}(\H)$) map to zero measurement change (i.e., $\H \c = \mathbf{0}$). 
In practical transmission \ac{se} formulations, the Jacobian is rank-deficient (reference/slack bus, measurement sparsity, etc.), and measurement/model mismatch, heterogeneous scaling, and numerical errors make the exact null space ill-defined for detection purposes. We therefore use the term {pseudo-null space} to denote the near-null directions of $\H$ (or of a reduced Jacobian) that produce only negligible measurement changes, i.e., vectors $\c_n$ with $\H \c_n \approx \mathbf{0}$ in the presence of realistic noise and modeling error. 
Attacks constructed along these directions can cause significant state bias while leaving residuals (and many statistical detectors) unchanged \cite{mukherjee2025false} (see \Cref{sec:results:fdia,sec:state_level_analysis}).

\subsection{Practical Bridge: Model null space vs.\ Measurement-derived Pseudo-Null Space.}  
In many modern operational settings, the true Jacobian $\H$ is partially unknown, outdated, or effectively time-varying (outsourced preprocessing, heterogeneous \ac{pmu} scaling, \ac{ai}-in-the-loop dispatch, etc.). 
Consequently, computing $\mathrm{Null}(\H)$ may be infeasible or unreliable. 
Therefore, in practice, we infer an empirical pseudo-null space from a short history of measurements using low-rank subspace extraction (e.g., \ac{svd}/\ac{pca}). 
This measurement-derived pseudo-null space encodes the {empirical} geometric directions that carry little measurement magnitude. The detection principle used in this paper is to (i) make the measurement-derived pseudo-null extraction physically meaningful through appropriate normalization and standardization, and then (ii) monitor the projection into that inferred pseudo-null subspace: anomalous increases indicate geometric misalignment consistent with a low-magnitude, null space aligned \ac{fdia}. 
This strategy therefore implements the conceptual comparison between the physical invariance implied by $\H$ and the behavior observed in data, without requiring explicit, perfect knowledge of $\H$.

Although the discussion above focuses on the \ac{dc} model for analytical transparency, the same principle extends to \ac{ac} state estimation.  
In the \ac{ac} case, measurements follow $\z = \h(\x) + \noise$, where $\h(\cdot)$ encodes nonlinear power-flow relationships between voltage magnitudes and phase angles.  
The corresponding \ac{wls} estimator is formulated as
\begin{equation}
    \ex = \arg\min_{\x} (\z - \h(\x))^\top \eR^{-1} (\z - \h(\x)),
\end{equation}
which requires iterative numerical solutions but preserves the same geometric interpretation.
Small perturbations aligned with the local column space of the Jacobian $\mathbf{J}(\x)$ can still induce significant state deviations while evading residual-based tests.

\section{\acf{ping} Framework}
\label{sec:approach}

\begin{figure*}[t]
\centering
\includegraphics[width=0.9\linewidth]{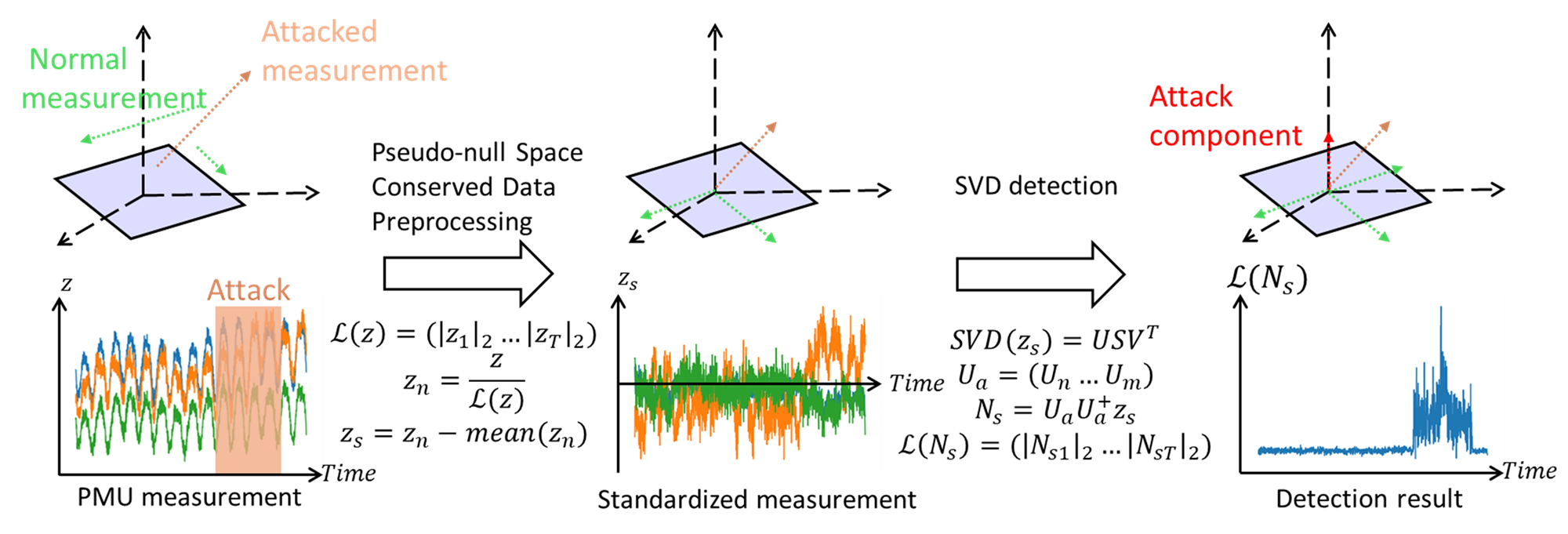}
\caption{Overview of the proposed \ac{ping} framework. The method first performs pseudo-null space-conserved preprocessing and then applies low-rank subspace extraction to isolate stealthy FDIA components.}
\label{fig:big_det}
\end{figure*}

Modern grids operate under highly dynamic, heterogeneous, and \ac{ai}-assisted conditions.  
Data streams originate from diverse sources, including \acp{pmu}, smart meters, and cloud analytics, while adaptive control continually reshapes operating points and effective Jacobians.  
These evolving data manifolds invalidate the fixed-model assumptions underlying most \ac{fdia} detection methods, making both residual-based and learning-based detectors unreliable.  
This section introduces the \ac{ping} framework, which targets low-magnitude, geometry-preserving \acp{fdia} that exploit these null space structures.  
We first formulate the pseudo-null space \ac{fdia} in \Cref{sec:attack}, then describe the invariant detection principle and its practical realization through the \ac{pscp}–\ac{svd} procedure in \Cref{sec:detection}.

\subsection{Pseudo-Null Space \ac{fdia}}
\label{sec:attack}
We assume two capabilities commonly attributed to adversaries:
(1) the ability to modify arbitrary measurement values, e.g., by compromising the data flow; and
(2) the ability to estimate partial or approximate system parameters using historical data.
Prior works show that even without explicit topology access, the system Jacobian $\H$ can be inferred from spatiotemporal correlations~\cite{costilla2022attack}.

Classic stealthy \acp{fdia}~\cite{liu2011false} exploit the inferred Jacobian to inject $\a = \H\c$, where $\c$ is any vector in the state space.
Because $\S\H = \mathbf{0}$, the resulting residual $\|\S(\z+\H\c)\|_2^2 = \|\S\z\|_2^2$ remains unchanged, enabling the attack to bypass \acp{bdd}.
However, such model-based \acp{fdia} distort the geometry of the measurement manifold, producing persistent temporal patterns detectable by machine-learning methods~\cite{zu2024self}.
\rev{Here, the measurement manifold in the \ac{dc} model refers to the set of all physically feasible measurements generated by the DC state-estimation model. Since the DC measurement function is linear, this set is
\(
\mathcal{H}_{\mathrm{DC}}
=
\{\H\x \mid \x\in\mathbb{R}^{\dimx}\}.
\)
Because the number of measurements is typically larger than the number of states, \(\mathcal{H}_{\mathrm{DC}}\) is a low-dimensional linear subspace embedded in the measurement space \(\mathbb{R}^{\dimz}\). Normal measurements lie near this subspace, with deviations mainly caused by measurement noise.
}
In this paper, we consider a new class of {pseudo-null space attacks} that regain stealth under non-stationary and heterogeneous conditions by manipulating measurements along directions that remain approximately invariant even as the system Jacobian drifts.

\rev{Under the \ac{dc} model, $\H$ is rank-deficient due to the reference bus.
After removing the reference-bus column, we obtain a reduced Jacobian $\H_r$.
We define the {$\tau$-pseudo-null space} of $\H_r$ as
\begin{equation}
\mathcal N_\tau(\H_r)
=
\left\{
\c_n \in \mathbb R^{\dimx-1}
:
\|\H_r \c_n\|_2 \le \tau \|\c_n\|_2
\right\},
\label{eq:pseudo-null-def}
\end{equation}
where $\tau>0$ is a prescribed tolerance.
Equivalently, if $\H_r = U\Sigma V^\top$, then $\mathcal N_\tau(\H_r)$ is spanned by the right singular vectors associated with singular values below $\tau$.}

\rev{
An attacker constructs a stealth perturbation $\a = \H_r\c$ with $\c = \c_n + \c_e$, where $\c_n\in\mathcal N_\tau(\H_r)$ and $\c_e \sim \mathcal{N}(\mathbf{0}, \sigma_c^2I_{n-1})$ adds small Gaussian fluctuations.
Because $\|\H_r\c_n\|_2 \le \tau \|\c_n\|_2$, the deterministic component induced by $\c_n$ is small, so the effective perturbation $\a$ becomes statistically close to \ac{pmu} noise when $\tau$ is properly chosen.}
By tuning the variance of $\c_e$ to half the measurement-noise variance, the injected modification hides within nominal uncertainty while altering the state estimate.
Thus, the attack reshapes state estimation without triggering either \ac{bdd} alarms or data-driven detectors that depend on statistical stationarity (see \Cref{sec:results:fdia}, \Cref{fig:detection_comparison2,fig:detection_comparison}).

\rev{
The parameter $\tau$ controls the tradeoff between stealth and attack flexibility: smaller $\tau$ yields more stealthy but less diverse perturbations, while larger $\tau$ enlarges the feasible attack subspace at the cost of increased detectability.
In practice, $\tau$ can be chosen by singular-value truncation or relative to the measurement noise level.
}

\subsubsection{Extension to \ac{ac} model}
\label{sec:attack:ac}

\rev{
The pseudo-null space concept naturally extends to the \ac{ac} model governed by nonlinear measurement equations $\z=\h(\x)+\noise$, where \(\h(\cdot)\) maps the system state, including voltage magnitudes and phase angles, to nonlinear power-flow measurements such as active/reactive power injections and branch flows.
In this setting, the corresponding measurement manifold is the set of all physically feasible \ac{ac} measurements,
\(
\mathcal{H}_{\mathrm{AC}}
=
\{\h(\x)\mid \x\in\mathbb{R}^{\dimx}\}.
\label{eq:ac_manifold}
\)
Unlike the \ac{dc} case, where the measurement manifold is a linear subspace generated by \(H\), \(\mathcal{H}_{\mathrm{AC}}\) is generally a nonlinear low-dimensional manifold embedded in the measurement space \(\mathbb{R}^{\dimz}\). Around a given operating point \(\x\), its local geometry is characterized by the Jacobian \(\boldsymbol{J}(\x)\).
In this case, a stealthy perturbation should satisfy
\begin{equation}
\h(\x+\c_n)-\h(\x)\approx\mathbf{0}.
\end{equation}
Linearizing $\h(\cdot)$ around $\x$ gives $\h(\x+\c_n)\approx \h(\x)+\boldsymbol{J}(\x)\c_n$, where $\boldsymbol{J}(\x)$ is the local Jacobian.
Accordingly, we define the local $\tau$-pseudo-null space as
\begin{equation}
\mathcal N_\tau(\boldsymbol J(\x))
=
\left\{
\c_n
:
\|\boldsymbol J(\x)\c_n\|_2 \le \tau \|\c_n\|_2
\right\}.
\end{equation}
Choosing $\c_n\in\mathcal N_\tau(\boldsymbol J(\x))$ ensures that the first-order measurement variation remains small.
Since $\boldsymbol{J}(\x)$ evolves with system conditions, this construction explains why small, physically consistent perturbations remain hard to detect in realistic \ac{ac} settings.
}

\rev{
We emphasize that the proposed pseudo-null space FDIA represents a worst-case threat model for evaluating detector robustness. Specifically, the attacker is assumed to have access to an abundant volume of normal historical measurements and can use these data to learn or approximate the measurement manifold and its pseudo-null directions. Similar assumptions are common in blind FDIA studies, where the attacker does not require exact system parameters but can infer attack directions from measurement data. The effectiveness of such attacks depends on data availability and data quality: with fewer normal samples, partial meter access, inaccurate historical data, or distribution drift, the learned measurement manifold becomes less accurate and the attack performance is expected to degrade. Therefore, the attack proposed in Section~\ref{sec:attack} should be interpreted as a stress-test scenario for evaluating whether the detection method introduced in Section~\ref{sec:detection} can identify low-magnitude, physically consistent FDIAs under strong adversarial conditions.
}

\subsection{Detection via Pseudo-Null Space Projection}
\label{sec:detection}

Traditional \ac{bdd} schemes~\cite{liu2011false,hug2012vulnerability} cannot detect attacks aligned with $\mathcal{H}_{\mathrm{DC}}$.  
Machine-learning detectors, such as deep neural and sequence models~\cite{tahar2022machine,sa2023false}, can track complex temporal relationships but fail when adversaries exploit pseudo-null directions that mimic legitimate noise.  
This presumed advantage of attackers calls for development of a detection mechanism that operates without assuming fixed or fully known measurement-space models.

\rev{
We note, for clarity, that 
\ac{pscp} preserves the geometric structure that links $\mathcal{N}(\H^\top)$ with its data-driven counterpart (\Cref{prop:null_conservation}), so the \ac{svd}-based $\U_a$ serves as an empirical proxy for the physical null space. 
Throughout the paper, "alignment" refers to this preserved correspondence, rather than a direct subspace comparison.
}

The proposed \ac{ping} framework dynamically identifies the pseudo-null space from streaming data and monitors its physical consistency without requiring explicit system models.
As illustrated in \Cref{fig:big_det}, \ac{ping} performs two main steps:  
(1) \textbf{\acf{pscp}} to mitigate scaling bias and floating-point artifacts, and  
(2) \textbf{low-rank subspace extraction} via \ac{svd} to isolate hidden null space deviations.

Let $\Z=(\z_1,\ldots,\z_\time)\in\R^{\dimz\times \time}$ denote the time-series measurement matrix.  
Applying \ac{svd} yields $\Z=\U\S\V^\top$, where $\U=(\u_1,\ldots,\u_{\dimz})$ and $\S$ is diagonal.  
The dominant $(\dimx-1)$ left singular vectors $\U_r=(\u_1,\ldots,\u_{\dimx-1})$ represent the physical row space, while the remaining vectors $\U_a=(\u_{\dimx},\ldots,\u_{\dimz})$ span the pseudo-null subspace.  
Orthogonal projection isolates the null space component:
\begin{equation}\label{eq:pseudo-test}
\N = \U_a \U_a^+ \Z,
\end{equation}
where $\U_a^+ = (\U_a^\top \U_a)^{-1}\U_a^\top$ is the pseudoinverse.

Direct projection on raw $\Z$ amplifies numerical errors due to heterogeneous \ac{pmu} scaling.  
Therefore, \ac{ping} introduces a normalization step, expressing $\Z$ as $\Z_n = \Z \text{diag}({\L(\Z)})^{-1}$, where $\L(\Z)=(\|\z_1\|_2,\ldots,\|\z_t\|_2)$.  
Projection on normalized data  yields:
\begin{equation}\label{eq:pseudo-test-norm}
\N_n = \U_a \U_a^+ \Z_n.
\end{equation}

\rev{Before establishing the separability result of \Cref{prop:opleak}, we state the key property of {pseudo-null  space conserved} preprocessing.

\begin{prop}[null space Conservation]
\label{prop:null_conservation}
Let $\z_i = \H \x_i + \noise_i$ be the clean measurement at time 
$\time$, and let $\U_a \U_a^{+}$ be the orthogonal projector onto 
the left null space $\mathcal{N}(\H^\top)$. The two operations in 
\ac{pscp}, per-sample $\ell_2$ normalization $\z_i \mapsto \z_i / 
\|\z_i\|_2$ and $\z_i \mapsto \z_i - \mathrm{mean}(\Z)$, both preserve $\mathcal{H}_{\mathrm{DC}}$.
\begin{equation}
\U_a \U_a^{+} \z_i = \mathbf{0} 
\Longleftrightarrow 
\U_a \U_a^{+} \tilde{\z}_i = \mathbf{0},
\end{equation}
where $\tilde{\z}_i$ is the preprocessed sample.
\end{prop}

\begin{proof}
{(i) Normalization.} For $\z_i = \H \x_i \in \mathcal{H}_{\mathrm{DC}}$, we have $\z_i / \|\z_i\|_2 = \H (\alpha_i \x_i)$ with 
$\alpha_i = 1/\|\z_i\|_2 > 0$. Thus $\alpha_i \z_i \in \mathcal{H}_{\mathrm{DC}}$ and $\U_a \U_a^{+} (\alpha_i \z_i) = \mathbf{0}$.

{(ii) Centering.} $\mathrm{mean}(\Z) = \H\,\mathrm{mean}
(\mathbf{X}) + \mathrm{mean}(\mathbf{E})$. Since $\mathrm{mean}
(\mathbf{E}) \to \mathbf{0}$ asymptotically, $\mathrm{mean}(\Z) 
\in \mathcal{H}_{\mathrm{DC}}$. By linearity, $\U_a \U_a^{+}(\z_i - 
\mathrm{mean}(\Z)) = \mathbf{0}$.

In contrast, the per-feature standardization $z$-score satisfies
$z_{i,t} \mapsto (z_{i,t} - \mu_i)/\sigma_i$ which corresponds to a 
diagonal rescaling $\mathbf{D}^{-1}\z_i$ with $\mathbf{D} = 
\mathrm{diag}(\sigma_i)$. Since $\mathbf{D}^{-1}\H \neq \H$ in 
general, $\mathbf{D}^{-1}\mathcal{H}_{\mathrm{DC}} \neq \mathcal{H}_{\mathrm{DC}}$, so $\U_a \U_a^{+} (\mathbf{D}^{-1}\z_i) \neq \mathbf{0}$. 
The empirical pseudo-null subspace extracted from $z$-scored data therefore 
does not coincide with $\mathcal{N}(\H^\top)$.
\end{proof}

\Cref{prop:null_conservation} shows that preprocessing preserves the separation between $\mathcal{H}_{\mathrm{DC}}$ and its 
orthogonal complement, so the \ac{svd}-derived $\U_a$ keeps isolating noise from signal even when $\H$ is unknown or 
drifting. This establishes the physical consistency of \ac{ping}.}

\paragraph{Standardization for Improved Separation}  
Normalization alone does not remove dataset-wide bias.  
\Cref{prop:opleak} shows that performing orthogonal projection on standardized measurements $\Z_s = \Z_n - \mathrm{mean}(\Z_n)$ enhances separability between attacked and normal samples.

\rev{
\begin{prop}[Operating-Point Leakage Removal]
\label{prop:opleak}
Let $\bar{\U}_a$ be the noise-free pseudo-null basis satisfying $\bar{\U}_a\bar{\U}_a^{+}\H = \mathbf{0}$, and let $\U_a$ denote the empirical pseudo-null basis estimated from a finite, noisy window $\Z \in \R^{\dimz \times \time}$ at noise level $\sigma$.
Then the residual leakage of the operating-point baseline through $\U_a$ satisfies
\begin{equation}
\bigl\|\U_a\U_a^{+}\,\H\bar{\x}\bigr\| = \mathcal{O}\!\left(\tfrac{\sigma}{\sqrt{\time}}\,\|\H\bar{\x}\|\right),
\label{eq:leakbound}
\end{equation}
where $\bar{\x} = \tfrac{1}{\time}\sum_{i=1}^{\time}\x_i$.
After \ac{pscp} centering, the projected sample reduces to
\begin{align}
    &\U_a\U_a^{+}\bigl(\z_t - \mathrm{mean}(\Z)\bigr) = \\&\U_a\U_a^{+}\bigl(\H\Delta\x_t + \a_t - \bar{\a} + \noise_t - \bar{\noise}\bigr),
\label{eq:centered}
\end{align}
with $\Delta\x_t = \x_t - \bar{\x}$, $\bar{\a} = \tfrac{1}{\time}\sum_i \a_i$, and $\bar{\noise} = \tfrac{1}{\time}\sum_i \noise_i$.
The persistent leakage of \eqref{eq:leakbound} is canceled.
\end{prop}
\begin{proof}
\textit{Bound \eqref{eq:leakbound}.}
Let $\bar{\mathbf{P}}_a = \bar{\U}_a\bar{\U}_a^{+}$ denote the noise-free projector, which satisfies $\bar{\mathbf{P}}_a\H = \mathbf{0}$ by definition.
\\
By the Davis--Kahan theorem~\cite{yu2015useful}, the empirical projector $\U_a\U_a^{+}$ deviates from $\bar{\mathbf{P}}_a$ by:
\\
Assume a constant spectral gap $\delta_{\text{gap}} > 0$ separating the row-space singular values of $\H$ from the pseudo-null directions,
i.e.,
\begin{equation}
\delta_{\text{gap}} =\sigma_{\dimx-1}(\H) - \sigma_{\dimx}(\H) >0,
\label{eq:gap}
\end{equation}
where $\sigma_{k}(\H)$ denotes the $k$-th singular value of $\H$
in decreasing order. Applying the Davis--Kahan $\sin\Theta$ theorem to the sample covariance $\hat{\boldsymbol{\Sigma}} = \tfrac{1}{\time}\Z\Z^{\top}$,
and absorbing the effective dimension into the constant, yields
\begin{equation}
\bigl\|\U_a\U_a^{+} - \bar{\mathbf{P}}_a\bigr\|
=\mathcal{O}\!\left(\frac{\sigma}{\sqrt{\time}\,\delta_{\text{gap}}}\right)
=\mathcal{O}\!\left(\frac{\sigma}{\sqrt{\time}}\right),
\label{eq:dk-bound}
\end{equation}
where the second equality treats $\delta_{\text{gap}}$ as a constant.
Since $\bar{\mathbf{P}}_a\H\bar{\x} = \mathbf{0}$,
\begin{align}
    &\bigl\|\U_a\U_a^{+}\H\bar{\x}\bigr\| =
    \\&\bigl\|(\U_a\U_a^{+} - \bar{\mathbf{P}}_a)\H\bar{\x}\bigr\| \leq \bigl\|\U_a\U_a^{+} - \bar{\mathbf{P}}_a\bigr\|\,\|\H\bar{\x}\|,
\end{align}
which yields \eqref{eq:leakbound}.
Substituting $\mathrm{mean}(\Z) = \H\bar{\x} + \bar{\a} + \bar{\noise}$ into $\z_t - \mathrm{mean}(\Z)$ and applying $\U_a\U_a^{+}$ gives \eqref{eq:centered} directly.
\end{proof}
The dominant $\mathcal{O}(\sigma/\sqrt{\time})\,\|\H\bar{\x}\|$ leakage is removed because $\bar{\x}$ is persistent across the window.
The remaining term $\U_a\U_a^{+}\H\Delta\x_t$ is governed by per-sample state variation $\|\Delta\x_t\|$ rather than $\|\H\bar{\x}\|$, and is therefore much smaller in regimes where the operating point varies slowly relative to the window length.
\\
\begin{corollary}[Asymmetric Subspace-Estimation Requirement]
\label{cor:asym}
Suppose the attacker estimates the row space of $\H$ from a noisy window of length $\time_a$ at noise level $\sigma_a$, and constructs an attack $\a$ contained in the estimated row space. 
\\
Decomposed against the true $\H$, write $\a = \a^{\mathrm{row}} + \a^{\mathrm{null}}$ with $\a^{\mathrm{row}} \in \mathcal{R}(\H)$ and $\a^{\mathrm{null}} \in \mathcal{N}(\H^{\top})$. 
Then
\begin{equation}
\bigl\|\a^{\mathrm{null}}\bigr\| = \mathcal{O}\!\left(\tfrac{\sigma_a}{\sqrt{\time_a}}\right)\|\a\|.
\label{eq:anull}
\end{equation}
\ac{pscp} centering removes the persistent row-space leakage of $\a^{\mathrm{row}}$ but does not cancel $\a^{\mathrm{null}}$.
For the attack to remain below the defender's noise floor $\sigma_d$, the attacker must satisfy
\begin{equation}
\frac{\sigma_a}{\sqrt{\time_a}} \lesssim \frac{\sigma_d}{\|\a\|}.
\label{eq:asymbound}
\end{equation}
The defender's analogous subspace-estimation error is canceled by Proposition~\ref{prop:opleak}, while the attacker has no such cancellation.
Hence the attacker requires strictly more precise subspace estimation than the defender.
\end{corollary}

\begin{proof}
Let $\bar{\mathbf{P}}_a$ denote the true projector onto $\mathcal{N}(\H^{\top})$ and let $\mathbf{P}_a^{\mathrm{att}}$ denote the attacker's empirical counterpart.
By construction $\mathbf{P}_a^{\mathrm{att}}\a = \mathbf{0}$, so

\begin{align}
    \bigl\|\a^{\mathrm{null}}\bigr\| &= \bigl\|\bar{\mathbf{P}}_a\a\bigr\| 
    \\&= \bigl\|(\bar{\mathbf{P}}_a - \mathbf{P}_a^{\mathrm{att}})\a\bigr\| \leq \bigl\|\bar{\mathbf{P}}_a - \mathbf{P}_a^{\mathrm{att}}\bigr\|\,\|\a\|,
\end{align}

and Davis--Kahan gives $\|\bar{\mathbf{P}}_a - \mathbf{P}_a^{\mathrm{att}}\| = \mathcal{O}(\sigma_a/\sqrt{\time_a})$(under the same gap assumption as in \cref{prop:opleak}), yielding \eqref{eq:anull}.
For an attack persistent over $k$ of $\time$ samples, $\bar{\a}^{\mathrm{null}} \approx (k/\time)\,\a^{\mathrm{null}}$, so the centered residual contributes $(1 - k/\time)\,\a^{\mathrm{null}}$, which must remain below $\sigma_d$.
Dropping the $(1 - k/\time)$ factor yields \eqref{eq:asymbound}.
\end{proof}
}

The standardized projection is therefore:
\begin{equation}\label{eq:pseudo-test-norm2}
\N_s = \U_a \U_a^+ (\Z_n - \mathrm{mean}(\Z_n)).
\end{equation}
The corrected pseudo-null space components $\N_S=(\N_{s1},\ldots,\N_{s\time})$ are then evaluated via their magnitudes $\L_N=(\|\N_{s1}\|_2,\ldots,\|\N_{s\time}\|_2)$, with separation threshold $\mathcal{T}$ computed by the minimum cross-entropy criterion~\cite{li1993minimum,li1998iterative}:
\begin{equation}\label{eq:cross_entropy_threshold}
\mathcal{T} = \threLi(\L_N).
\end{equation}
If $\|\N_{si}\|_2 \geq \mathcal{T}$, the corresponding measurement is flagged as attacked.


\rev{\paragraph{Importance of Preprocessing}
By \Cref{prop:null_conservation}, \ac{pscp} preserves the decomposition $\z_i = \H\x_i + \noise_i$: the $\H\x_i$ ($\noise_i$) component 
remains within (outside) $\mathcal{H}_{\mathrm{DC}}$.
Projection via $\U_a \U_a^{+}$ therefore maintains noise-signal separation after preprocessing.
An attack that injects perturbations outside $\mathcal{H}_{\mathrm{DC}}$ produces anomalous 
projections in \cref{eq:pseudo-test-norm2}.
In contrast, per-feature standardization mixes the two components.}

\rev{
\paragraph{Computational Complexity} 
\ac{ping} decomposes into an offline estimation step and an online single-frame test.
Given a historical window $\Z\in\R^{\dimz\times \time}$, computing the \ac{svd} and the projector $\mathbf{P}_a = \U_a\U_a^{+}$ costs $\mathcal{O}(\dimz^2\time)$.
At each new frame $\z_i$, only $\mathcal{O}(\dimz(\dimz-\dimx+1))$ operations are required, hence the per-frame detection cost is independent of the window length $\time$.
As more measurements become available, the empirical estimate of the pseudo-null space 
$\U_a$ becomes increasingly accurate, and the offline step can be replaced by an incremental \ac{svd} \cite{brand2006fast} that progressively refines $\U_a$ with each incoming sample, at a per-update cost of 
$\mathcal{O}(\dimz \r + \r^{3})$.
A partial topology change (e.g., line switching or a single-element outage) which modifies only a few rows and columns of $\H$, is readily accommodated by the same incremental \ac{svd} per event.
A full \ac{svd} refresh is therefore needed only after significant topology reconfiguration, thereby making \ac{ping} suitable for real-time deployment.
}

\paragraph{Extension to \ac{ac} model.}

\rev{
The proposed \ac{ping} framework naturally extends to the \ac{ac} setting from a local measurement-manifold viewpoint. In the \ac{ac} model, the feasible measurement set is \(\mathcal{H}_{\mathrm{AC}}
=
\{\h(\x):\x\in\mathbb{R}^{\dimx}\},
\) which is generally a nonlinear manifold embedded in the measurement space \(\mathbb{R}^{\dimz}\). Around an operating point \(\x_0\), this manifold can be locally approximated by its tangent space, whose orientation is determined by the Jacobian \(\boldsymbol{J}(\x_0)\). Specifically,
\(
\h(\x_0+\Delta \x)
=
\h(\x_0)
+
\boldsymbol{J}(\x_0)\Delta \x
+
\mathcal{O}(\|\Delta \x\|_2^2).
\)
Thus, the accuracy of the local pseudo-null space approximation depends on the curvature of \(\mathcal{H}_{\mathrm{AC}}\) and the size of the operating-point variation within the analysis window. When the trajectory remains in a locally coherent operating region, the second-order term is small and the tangent-space approximation is reliable. When the trajectory spans strongly nonlinear or rapidly changing regions, the tangent space may rotate across samples, so a single \ac{svd}-derived subspace approximates a union of nearby local low-rank subspaces rather than one fixed flat subspace.
}

\rev{
This local low-rank structure is further reinforced by physical dependencies in \ac{ac} networks. In particular, the effective rank of \(\Z\) decreases when linear dependencies arise among currents and flows, especially along degree-2 series chains of injection-free buses. \Cref{lem:ac_cu} formalizes this property: all branch currents on a pure series chain are identical, reducing the measurement rank and strengthening the pseudo-null structure. Consequently, the same \ac{ping} principle applies to \ac{ac} measurements, while its subspace estimation accuracy depends on the local validity of the tangent-space approximation. This explains why \ac{ping} remains effective in standard IEEE \ac{ac} systems (see Section \ref{sec:results:ac}), while the resulting anomaly scores may exhibit higher variance than in the \ac{dc} case.
}

\begin{lemma}
\label{lem:ac_cu}
Under zero-injection and negligible-shunt assumptions, all branch currents in a pure series chain satisfy $I_{0,1}=I_{1,2}=\ldots=I_{l-1,l}$.
\end{lemma}

\begin{proof}
Applying \ac{kcl} at each internal node $i_k$ yields $I_{k-1,k}+I_{k+1,k}+I_{i_k}^{sh}=0$.  
With $I_{i_k}^{sh}=0$ and orientation $I_{k,k+1}=-I_{k+1,k}$, we have $I_{k,k+1}=I_{k-1,k}$.
\end{proof}
\rev{Note that this establishes exact linear dependence only for branch currents and that the broader low-rank structure observed in AC measurements (\cref{sec:results:ac}) reflects approximate dependencies that we characterize empirically rather than analytically.}
Because these current series chains are common in both transmission and distribution grids, pseudo-null dependencies appear naturally. 
\rev{This supports the applicability of \ac{ping} for nonlinear \ac{ac} environments.}
\rev{
\paragraph{Physical Interpretation of Null-Space Conservation}
\label{sec:physical_interpretation}
The Null-Space Conservation Proposition guarantees that \ac{pscp} preserves $\mathcal{H}_{\mathrm{DC}}$.
To clarify what this means physically, we identify the invariants encoded by $\mathcal{H}_{\mathrm{DC}}$ and contrast them with operational events that do alter it.
\\
In the \ac{dc} model, each row of $\H$ is a linear functional of the bus-angle vector $\x$ expressing a branch flow or bus injection through the linearized power-flow equations. Hence $\mathcal{H}_{\mathrm{DC}}$ is parameterized by  the following physical information: the network {topology}, the branch {reactances} $\{x_{ij}\}$, and the {linearized active power-balance equations}.
These are the invariants explicitly preserved by \ac{pscp}.
\ac{kcl} is implicitly preserved: \ac{kcl} is embedded in every injection row $P_i = \sum_{j \in \mathcal{N}(i)} (\theta_i - \theta_j)/x_{ij}$.
In the \ac{ac} setting, \Cref{lem:ac_cu} strengthens this picture: on zero-injection degree-2 chains, \ac{kcl} becomes an exact linear dependence, so \ac{kcl} is preserved explicitly along such chains.
\\
The same correspondence clarifies which events do change $\mathcal{H}_{\mathrm{DC}}$. 
Consider a 3-bus \ac{dc} system with branches $1$--$2$, $1$--$3$, $2$--$3$ and reactances $x_{12} = x_{13} =x_{23} = 0.1$~p.u. 
With $P_{ij} = (\theta_i - \theta_j)/x_{ij}$, the baseline measurement matrix is
\begin{equation}
\H = \frac{1}{0.1}
\begin{bmatrix}
1 & -1 & 0 \\
1 & 0 & -1 \\
0 & 1 & -1
\end{bmatrix},
\label{eq:H_baseline}
\end{equation}
whose range encodes the loop relation $P_{12} - P_{13} + P_{23} = 0$.
If the breaker on line $2$--$3$ opens, the corresponding row vanishes, the loop $1\!-\!2\!-\!3\!-\!1$ is removed, and
\begin{equation}
\H_{\text{open}} = \frac{1}{0.1}
\begin{bmatrix}
1 & -1 & 0 \\
1 & 0 & -1
\end{bmatrix},
\end{equation}
so $\mathrm{Range}(\H_{\text{open}}) \neq \mathcal{H}_{\mathrm{DC}}$.
Alternatively, keeping the topology fixed but updating $x_{12}$ from $0.1$ to $0.05$~p.u.\ yields $\widetilde{\H}$ with rows $[20,-20,0]$, $[10,0,-10]$, $[0,10,-10]$; the rank is unchanged but the loop relation becomes $0.5\,P_{12} - P_{13} + P_{23} = 0$, so $\mathrm{Range}(\widetilde{\H}) \neq \mathcal{H}_{\mathrm{DC}}$.
Topology and reactance changes thus alter $\mathcal{H}_{\mathrm{DC}}$ and must be re-identified, whereas \ac{pscp} keeps transformed measurements inside $\mathcal{H}_{\mathrm{DC}}$ and therefore preserves all of these invariants exactly.
}

\section{\label{sec:evaluation}Experiments}

The experiments validate that the proposed \ac{ping} allows accurate detection of low-magnitude, geometry-preserving \acp{fdia} by utilizing the misalignment between the physical null space implied by the system model and the measurement-derived pseudo-null subspace estimated from data.  
In what follows, we demonstrate how this implicit null space alignment manifests empirically across different grid sizes, noise conditions, and observability levels.

\subsection{Setup} 
We evaluate our approach on standard transmission grid test cases to ensure robustness and scalability. Specifically, we conduct experiments on the IEEE $14$-bus, $30$-bus, $57$-bus, and $118$-bus systems, which are widely used in power system research due to their varying network sizes and complexities. These test systems form a comprehensive benchmark, ranging from sparse small grids (e.g., $14$-bus) to highly interconnected large networks (e.g., $118$-bus). In each system, we generate time-series measurement data by solving \ac{dc} power flow equations using MATPOWER \cite{zimmerman2010matpower}. 
To enhance realism, our simulations incorporate real-world power profiles from Duquesne Light Company, which is located in Pittsburgh, Pennsylvania, USA. 
Variability is introduced by scaling the load and generation profiles with randomly selected loading parameters. 
Additionally, Gaussian white noise with a standard deviation of 0.005 p.u. \cite{xiao2024privacy} is injected into the measurements to account for sensor noise.

From the outset, unless stated otherwise, our \ac{fdia} scheme assumes that all \ac{pmu} devices are under attack, with the exception of Section~\ref{sec:partial_obs_analysis} where only a subset of measurements is accessible for both attacker and defender, and thoroughly analyze the feasibility of our \ac{fdia} scheme and its detection in each of these cases. \rev{By default, the pseudo-null space threshold is set to \(\tau=0.01\).} For the small random fluctuation in the state-perturbation vector, we set \(\sigma_c=10^{-1}\|c_n\|_2/\sqrt{n-1}\), so that the expected fluctuation energy remains approximately one order of magnitude smaller than the deterministic pseudo-null component.
In addition, it is important to note that our \ac{fdia} scheme alters the power flow measurements within the regime of the pseudo-null space, ensuring that the measurements do not significantly deviate, thereby retaining attack stealth and circumventing time synchronization effects.

\subsection{Pseudo-Null Space \ac{fdia}}
\label{sec:results:fdia}
We begin the demonstration of the main results by depicting how a pseudo-null space attack alters the physical states while keeping the measurement-space geometry nearly unchanged, a misalignment that \ac{ping} should detect.
\Cref{fig:detection_comparison2} illustrates the pseudo-null space \ac{fdia} effects within the time-series measurement data. 
During the attack period, the estimated states deviate by more than 5\%, confirming the attack's effectiveness and impact. Meanwhile, the residual error remains consistently low (approximately $5
\times 10^{-5}$), demonstrating that the attack successfully bypasses the \ac{bdd}. 

\begin{figure}[t]
    \centering
    \includegraphics[width=1\linewidth]{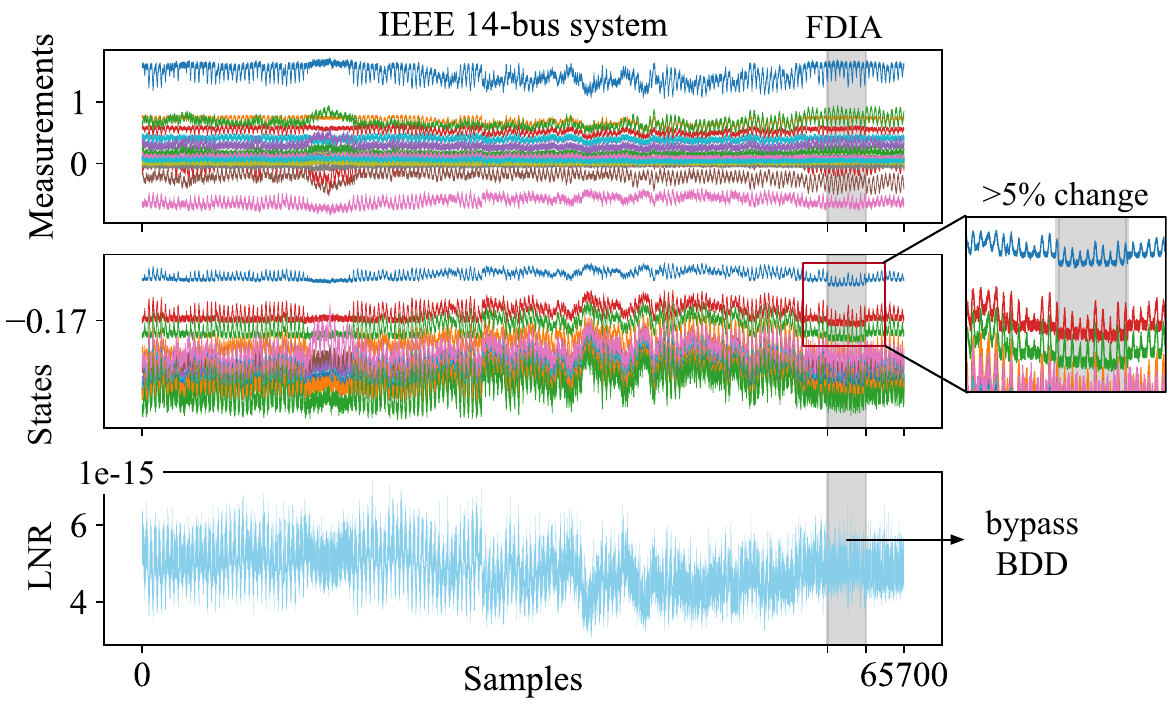}
    \vskip -0.1in
    \caption{The proposed FDIA on measurements (upper plot), induced impact on state (middle plot), and corresponding \ac{bdd} result (lower plot). Gray background marks the attack period. }
    \label{fig:detection_comparison2}
\end{figure}

\subsection{Comparing \ac{ping} with \acs{ml}-based Detectors}
\label{sec:results:detection}

Given the pseudo-null space \ac{fdia}, we further use \Cref{fig:detection_comparison} to compare our proposed detection mechanism to simple and advanced \ac{ml}-based \ac{fdia} detectors, including an Isolation Forest detector, an \ac{lstm}-based detector~\cite{deepSignalAnomalyDetector,MATLAB:R2024b}, a \ac{xtm}~\cite{baul2023xtm} and an \ac{ae}~\cite{wang2020detection}.
In Fig.~\ref{fig:detection_comparison}, the first observation is that the Isolation Forest detector completely fails across all grid scales, highlighting the effectiveness of the investigated attack. 
More advanced \ac{ml}-based detectors (\ac{lstm} and \ac{xtm}) perform comparably to \ac{ping} detecting the stealthy \ac{fdia} in smaller grids (i.e., the IEEE 14-bus system). 
However, in larger grids, both the \ac{lstm} and \ac{xtm} detectors struggle to localize the attack period and resulting in false alarms during normal operation after the attack. 
\rev{The \ac{ae} achieves better performance on larger grids, yet produces no detectable response on the IEEE 14-bus system.
We attribute the degradation of \ac{lstm} and \ac{xtm} to increased difficulty of training \ac{ml}-based models to recognize null space attacks as the system dimension (and thus the null space rank) grows.
}
\\
Unlike statistical data-driven detectors, \ac{ping} explicitly measures the magnitude of the measurement-derived pseudo-null component, allowing it to track subtle misalignment between the empirical and physical invariants. 
Consequently, it maintains reliable detection performance even as the system dimension and null space rank increase.

\begin{figure}[t]
    \centering
    \includegraphics[width=1\linewidth]{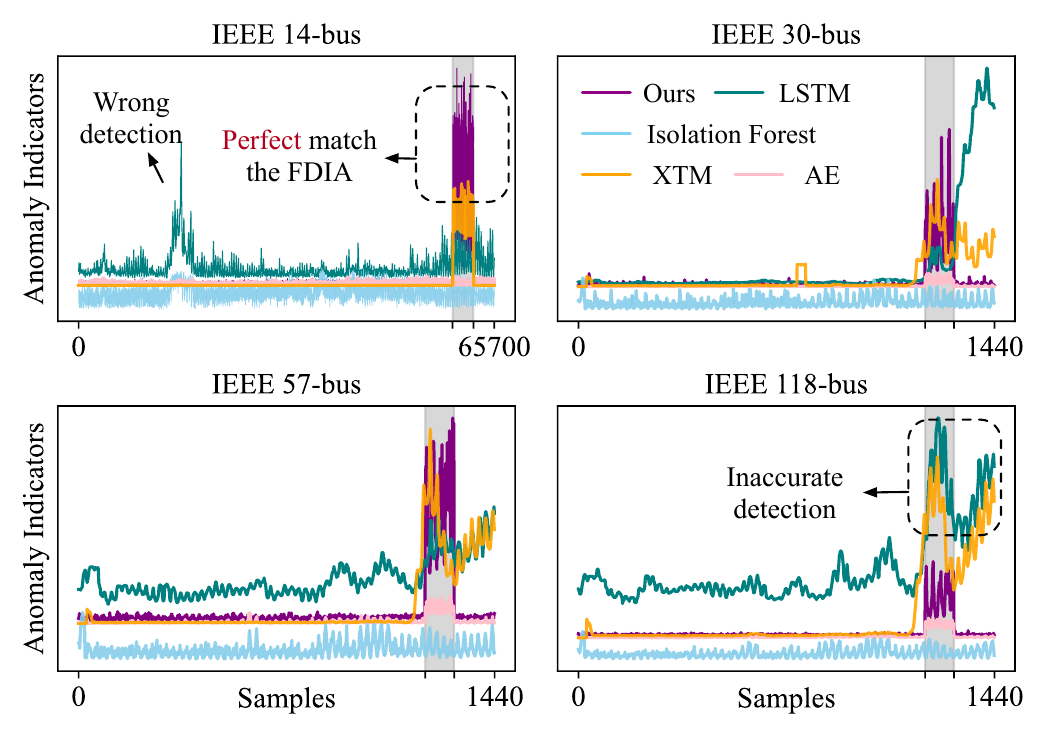}
    \vskip -0.1in
    \caption{The comparison of \ac{ping} and \ac{ml}-based detectors on the proposed FDIA in IEEE $14$, $30$, $57$, $118$-bus systems. The vertical gray-shaded region denotes the FDIA period.}
    \label{fig:detection_comparison}
\end{figure}

\subsection{Sensitivity to Noise}
\label{sec:sensitivity_to_noise}

To evaluate the robustness of the \ac{ping} framework under varying sensor noise levels, we test different \acp{std} of the Gaussian noise $\noise$, as shown in \Cref{fig:sensitivity_to_noise}. The results indicate that the proposed \ac{ping} detection method remains effective up to approximately $\text{std}\sim0.01$ p.u., which aligns with typical \ac{pmu} noise levels. 
With noise beyond this level, attack-induced anomalies become indistinguishable from normal operational variations, diminishing the ability of all detection methods, including \ac{ping}, to reliably identify \acp{fdia}. 
This analysis quantifies the noise tolerance of our approach and offers insights into its practical applicability across different sensing conditions.

\begin{figure}[t]
    \centering
    \includegraphics[width=1.0\linewidth]{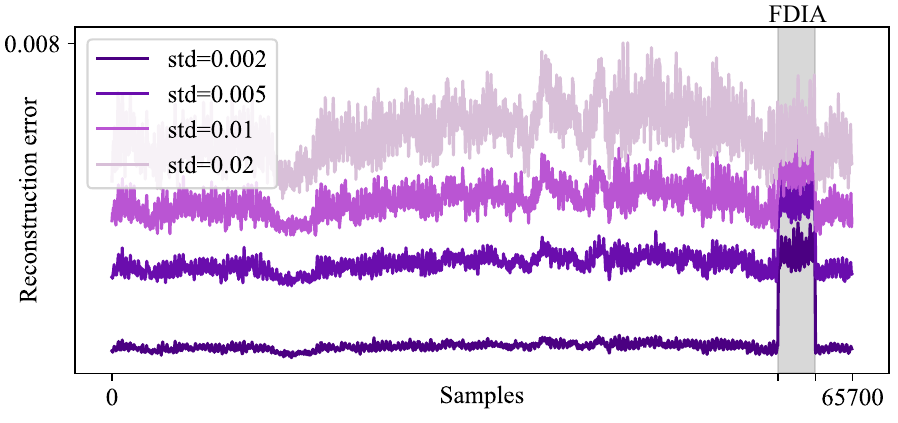}
    \caption{Sensitivity of \ac{ping} to the \acf{std} of the measurement noise.}
    \label{fig:sensitivity_to_noise}
\end{figure}

\subsection{\rev{Sensitivity to the Pseudo-null space Threshold $\tau$}}

\rev{
The tolerance threshold \(\tau\) in Eq. (\ref{eq:pseudo-null-def}) controls how closely a selected perturbation direction satisfies \(Hc\approx 0\): smaller values yield more stealthy near-null directions, while larger values admit a broader perturbation subspace.
To evaluate the sensitivity of PCNSA to \(\tau\), we vary \(\tau\) over \(\{10^{-2},5\times10^{-3},2\times10^{-3},10^{-3}\}\) and evaluate the resulting anomaly indicators across test systems. The results are shown in Fig.~\ref{fig:sensitivity_tau}. Across all four systems, the anomaly indicators remain clearly elevated during the FDIA interval for the tested values of \(\tau\), indicating that PCNSA is not overly sensitive to a single threshold choice. Although the absolute magnitude of the anomaly indicator varies with \(\tau\), the attack period remains distinguishable from the normal operating period.
}

\rev{
Several trends can be observed. 
Smaller values of \(\tau\) generally produce attack directions that are more tightly aligned with the pseudo-null space, leading to weaker measurement-space signatures and lower anomaly magnitudes. Larger values of \(\tau\) permit less precise null directions, amplifying the attack component in the measurement space and often yielding stronger anomaly indicators. 
The relative robustness of the detection pattern is consistent across different system sizes, including the IEEE 118-bus system. 
These observations confirm that \(\tau\) mainly controls a stealthiness--impact tradeoff in the attack construction, while the proposed PCNSA detector remains effective over a practical range of pseudo-null space thresholds.
}

\begin{figure}[h]
    \centering
    \vskip -0.1in
    \includegraphics[width=1\linewidth]{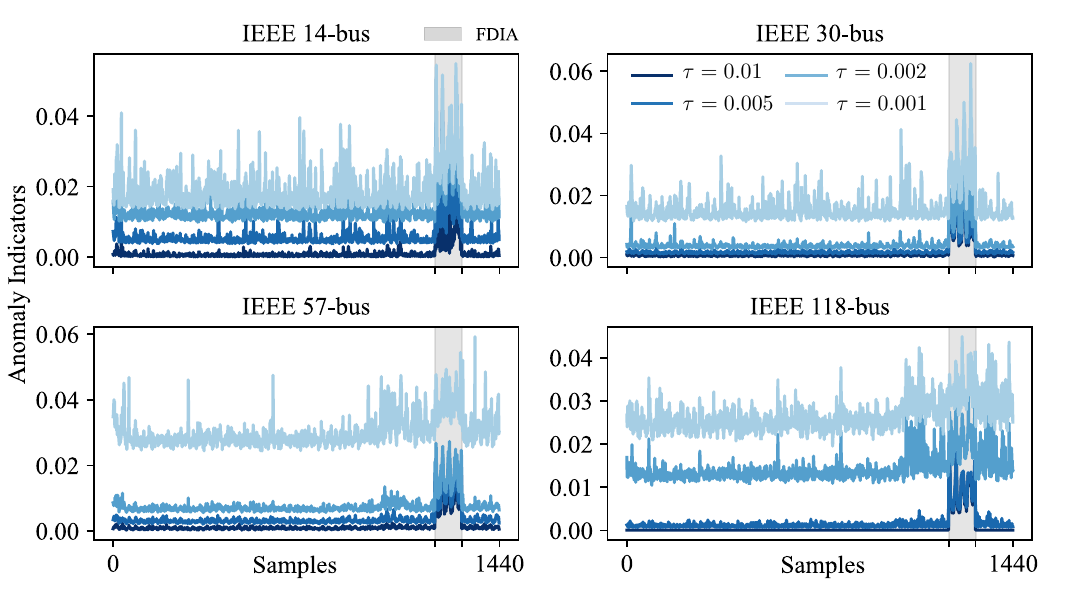}
    \vskip -0.1in
    \caption{\rev{Sensitivity analysis of PCNSA with respect to the pseudo-null space threshold \(\tau\) on IEEE 14-, 30-, 57-, and 118-bus systems. The gray shaded regions denote the FDIA periods. The anomaly indicators remain distinguishable during the attack interval across different values of \(\tau\), showing that PCNSA is robust to the pseudo-null space threshold selection.}}
    \label{fig:sensitivity_tau}
    \vskip -0.1in
\end{figure}




\subsection{Ablation Study}
\begin{figure*}[!htbp]
\includegraphics[width=1.0\linewidth]{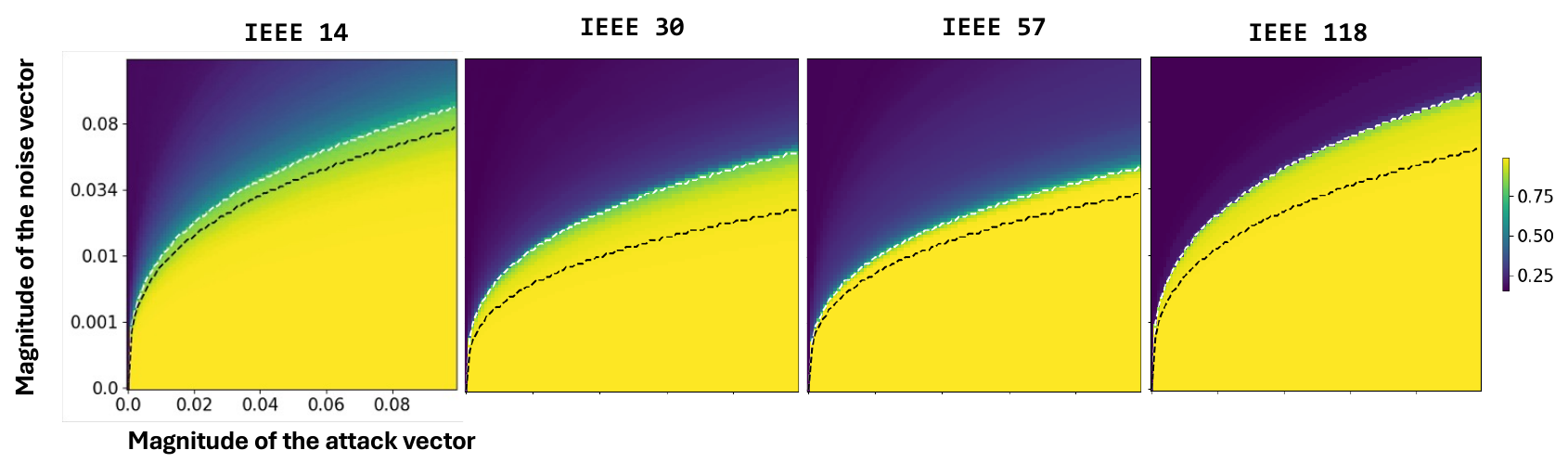}
\caption{Ablation analyses of the IEEE 14, 30, 57, 118-bus network systems, where the $F_1$ detection scores in the stealth-noise plane with data preprocessing (colored heat maps) and the $F_1$=0.8 curves with \ac{pscp} (dashed white) and without \ac{pscp} (dashed black) are shown for comparison.}
\label{fig:ablation_all_previous}
\end{figure*}
To demonstrate the effectiveness of the proposed \ac{ping}, we perform an ablation analysis in terms of prediction accuracy, i.e., we compare the results with and without data pre-processing in the presence of various stealth-noise combinations. 
\Cref{fig:ablation_all_previous} depicts \ac{ping} accuracy heat maps where the x/y axes represent the magnitude of the attack/noise vectors, respectively.
The color indicates the accuracy, which is computed using the $F_1$ score.
In addition, we show the $F_1=0.8$ curves with and without \ac{pscp}.
We observe that while our proposed detection method effectively covers a significant portion of the parameter space by providing precise detection of the attack period, there is still room for improvement, particularly at high noise levels where the attack is only partially detected or completely missed.
Conversely, attack evasion at the small attack vector region is much less of a concern, as smaller attack vectors are associated with negligible attack impact.
Finally, our proposed method achieves better detection coverage with \ac{ping} across large-scale systems, verifying the scalability of our approach.

\subsection{Evaluation of the Attack Impact}
\begin{figure}[h]
    \centering
    \includegraphics[width=0.8\linewidth]{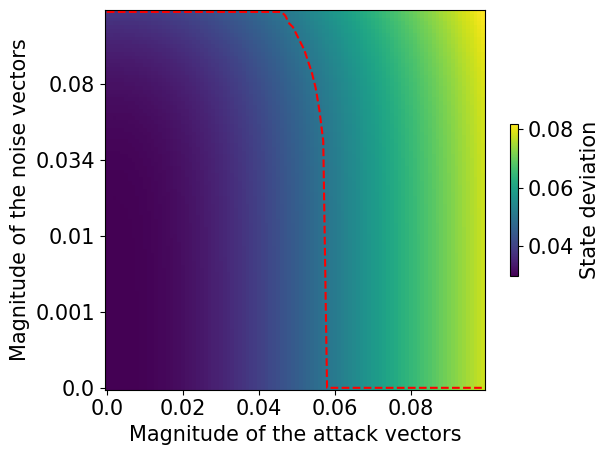}
    \caption{The impact of the studied \ac{fdia} on the state in terms of averaged state deviation (colored heat map) in the stealth-noise plane, for the IEEE 14 bus network system, where the 5\% state deviation (dashed red line) is shown for convenience.}
    \label{fig:state_impact_heatmap}
\end{figure}
\label{sec:state_level_analysis}
\Cref{fig:state_impact_heatmap} presents the attack impact on the state in the IEEE 14-bus system, quantified by the average state deviation percentage. 
The results indicate that the attack vector magnitude primarily determines the extent of state deviation, whereas measurement noise indirectly influences the impact by perturbing state estimation. 
Notably, by comparing \Cref{fig:ablation_all_previous,fig:state_impact_heatmap}, we observe that an attacker can induce a 5\% state deviation (marked by the red dashed line) while remaining undetected for large noise values ($\gtrsim 0.08$ p.u).

\subsection{Partial Observability Analysis}
\label{sec:partial_obs_analysis}
Next we analyze the \ac{fdia} impact and its detection by \ac{ping} while assuming that data from some meters is missing, e.g., due to communication disturbance. 

\paragraph{State deviation} \Cref{fig:IEEE14partialState} depicts the induced \ac{fdia} state deviation versus the number of missing meters for the IEEE 14 bus system. 
We see that the impact on the state is maximal when the number of observed measurements is $\sim$10 (50\%), underscoring the importance of effective \ac{fdia} detection under partial observability. 

\paragraph{Detection under partial observability}  
In \Cref{fig:IEEE14-118partial}, we plot the detection $F_1$ score as a function of the number of missing meters for the IEEE 14 (\Cref{fig:IEEE14-118partial}(a)) \rev{and 118 (\Cref{fig:IEEE14-118partial}(b))} bus systems.
In the IEEE 14 bus system, we see negligible $F_1$ degradation ($<0.2$) when missing fewer than six sensors, and the average $F_1$ is slightly below 0.71 when missing 10 sensors. 
\rev{In the IEEE 118 bus system, we see negligible $F_1$ degradation ($<0.2$) when missing fewer than 135 sensors.
This empirically demonstrates that larger grids exhibit higher tolerance to missing data, as larger systems often possess more redundant row structure, leading to increased measurement redundancy and improved rank stability.}

\begin{figure}
    \centering
    \vskip -0.1in
    \includegraphics[width=0.99\linewidth]{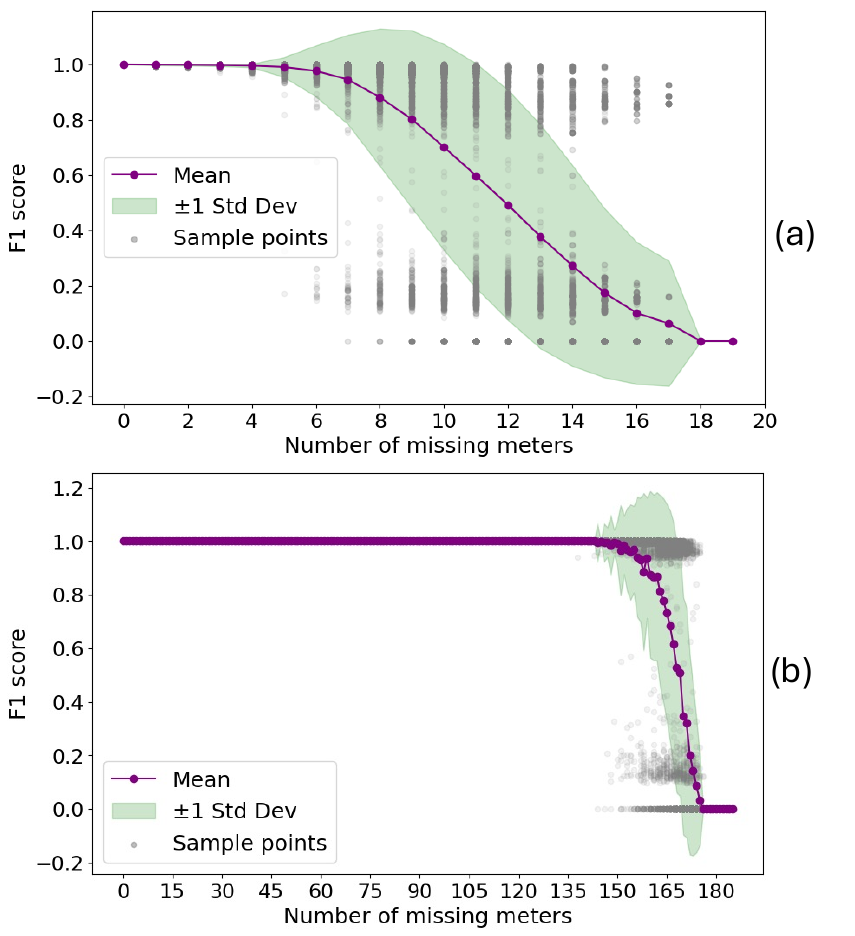}
    \vskip -0.1in
    \caption{\ac{ping} detection accuracy ($F_1$ score) of \ac{fdia} under partial observability conditions, i.e., as a function of the number of missed meters, in the IEEE-14 (a) and -118 (b) bus systems. The noise and full observability state deviation are set to 1e-2 and 0.05, respectively.}
    \label{fig:IEEE14-118partial}
    \vskip -0.1in
\end{figure}

\begin{figure}
    \centering
    \vskip -0.1in
    \includegraphics[width=0.99\linewidth]{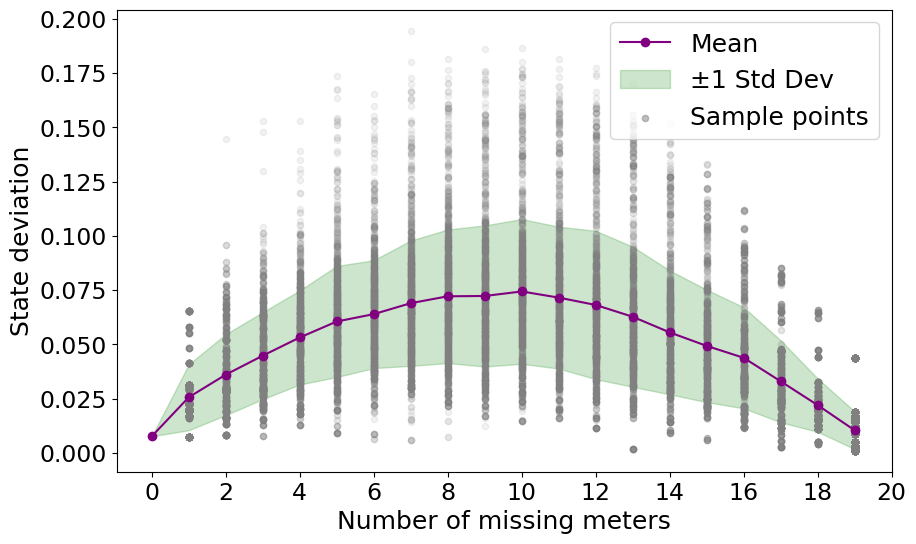}
    \vskip -0.1in
    \caption{State deviation of \ac{fdia} under partial observability conditions, i.e., as a function of the number of missed meters, in the IEEE-14 bus system. The noise and full observability state deviation are set to 1e-2 and 0.05, respectively.}
    \label{fig:IEEE14partialState}
    \vskip -0.1in
\end{figure}


\subsection{Extension to \ac{ac} Model}
\label{sec:results:ac}
To assess the generalizability of our framework beyond linear \ac{dc} model, we evaluate both \ac{fdia} and \ac{ping} in the \ac{ac} model without simplifying linearity assumptions. 
In this setting, $\z$ includes both active and reactive power flow measurements based on the \ac{ac} power flow equations. 
We compare our detection results to the \ac{ml}-based baselines in \Cref{fig:detection_comparison_AC}. 
Similar to the \ac{dc} setting, \ac{ml}-based anomaly detectors fail to reliably detect the attack. 
Specifically, the Isolation Forest produces no significant response, while the \ac{lstm} and \ac{xtm} detectors exhibit erratic behavior, with inconsistent localization and high false positive rates.
\rev{In contrast to the \ac{dc} setting, where the \ac{ae} achieves strong performance on larger grids, its accuracy degrades across all grid sizes under the \ac{ac} model.}
These results confirm the vulnerability of \ac{ml}-based anomaly detectors to pseudo-null space \acp{fdia}. 
Since such models operate purely on statistical patterns without leveraging the system’s physical structure, their performance is influenced by the consistency of the observed measurement trajectory. 
By locally adjusting $\boldsymbol{J}$ (see \Cref{sec:attack:ac}) to maintain the attack along the manifold, the adversary ensures that the perturbation remains statistically consistent with normal behavior, further underscoring the advantage of geometry-aware detectors like \ac{ping}.

\begin{figure}[t]
    \centering
    \includegraphics[width=1\linewidth]{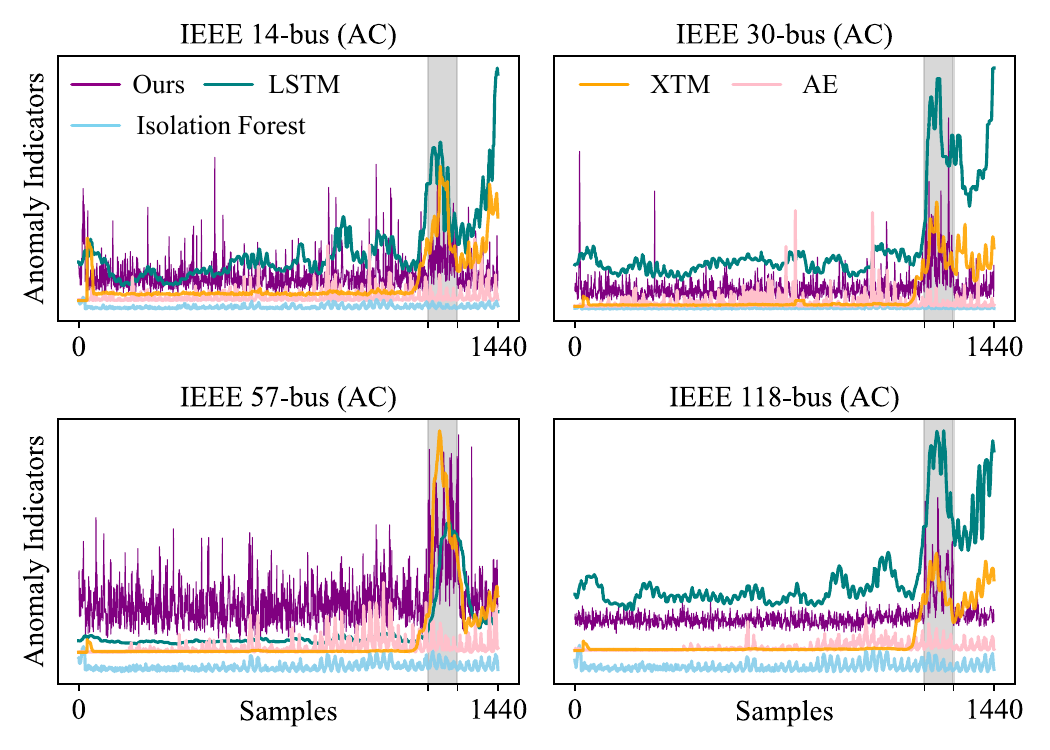}
    \caption{The comparison of \ac{ping} and \ac{ml}-based detectors on the proposed FDIA in IEEE $14$, $30$, $57$, $118$-bus systems under the \ac{ac} model.}
    \label{fig:detection_comparison_AC}
\end{figure}

In contrast, our \ac{ping} framework remains robust and consistently detects the attack period across the \ac{ac} model.
The attack interval remains clearly identifiable. 
Unlike in the \ac{dc} model (\Cref{fig:detection_comparison}), here the larger grids show more vivid anomalies.  

The performance clearly degrades compared to the \ac{dc} case with overall higher variance of the anomaly scores across all samples. 
This degradation arises from a key structural difference between \ac{dc} and \ac{ac} model: 
in the \ac{dc} model, the measurements lie on a low-dimensional flat subspace, resulting in a sharp drop in singular values after the true rank (see \Cref{fig:singularvalues} left). 
In the \ac{ac} model, series chains of degree-2 buses without injections enforce exact current equalities along the entire chain by \ac{kcl}.
Consequently, the time series of branch currents is perfectly collinear; the corresponding power flows are near-collinear, being simple modulations of the same chain current by slowly varying bus voltages.
These chain-level dependencies compress the intrinsic dimensionality of the measurement set beyond what is seen in \ac{dc}.

Globally, \ac{ac} measurements trace a nonlinear manifold $\z=\h(\x)$; as operating conditions drift, the local tangent $\boldsymbol{J}(\x)$ rotates, so data occupy a union of nearby low-rank subspaces rather than a single flat subspace.
This yields a lower effective rank but a gradual singular-value decay rather than a sharp collapse (see \Cref{fig:singularvalues} right).
We find that setting the \ac{svd} threshold $\tau$ to values between $10^{-1}$ and $10^{-2}$ yields robust and consistent performance across the \ac{ac} model.

\begin{figure}[h]
    \centering
    \includegraphics[width=1\linewidth]{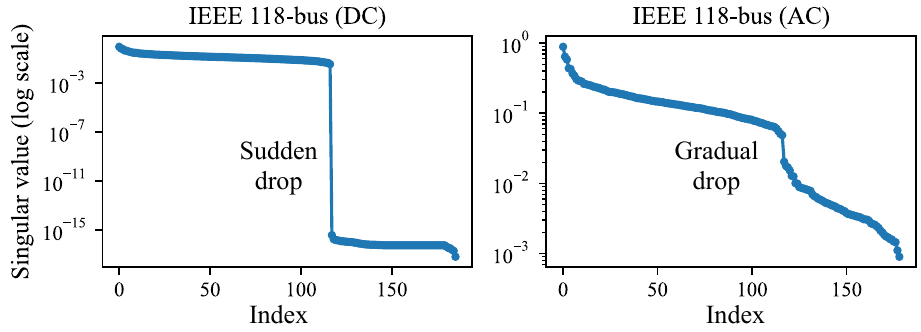}
    \caption{The singular values in the \ac{svd} process of power measurements in IEEE 118-bus system under the \ac{dc} (left) and \ac{ac} (right) models.}
    \label{fig:singularvalues}
\end{figure}

This consistency across nonlinear \ac{ac} manifolds further supports the central hypothesis of \ac{ping}: 
\textbf{stealth attacks can be revealed through deviations between the physically constrained null space and its empirical counterpart, even when the underlying Jacobian continuously evolves.}

To examine how zero-injection chains affect the effective rank of the measurement space, we isolated all degree-2 chains in the IEEE 118-bus system.
For each chain, we computed the least singular value $\sigma_{min}$ of its local measurement submatrix (formed from power-flow measurements on the two connected branches), and the average power-flow ratio, defined as:
\begin{equation}
 \bar{r}_{pf}=\frac{1}{\time}\sum_{i=1}^t\frac{|P_{from}(i)|}{|P_{to}(i)|},   
\end{equation}
indicating how balanced the two ends of the chain are over time.
As shown in Fig. \ref{fig:chain_singular}, most chains exhibit moderate singular values $10^{-3}<\sigma_{min}<1$ reflecting independent flow behavior.
However, a small cluster near $\bar{r}_{pf}\approx1$ has extremely small singular values $\sigma_{min}\approx10^{-6}$ marked by the red triangles.
These correspond precisely to zero-injection chains.
\rev{This empirical evidence is consistent with such chains that create low-rank dependencies in the \ac{ac} measurement manifold, explaining why the effective rank of \ac{ac} data is lower than that of the \ac{dc} model.}

\begin{figure}
    \centering
    \includegraphics[width=0.9\linewidth]{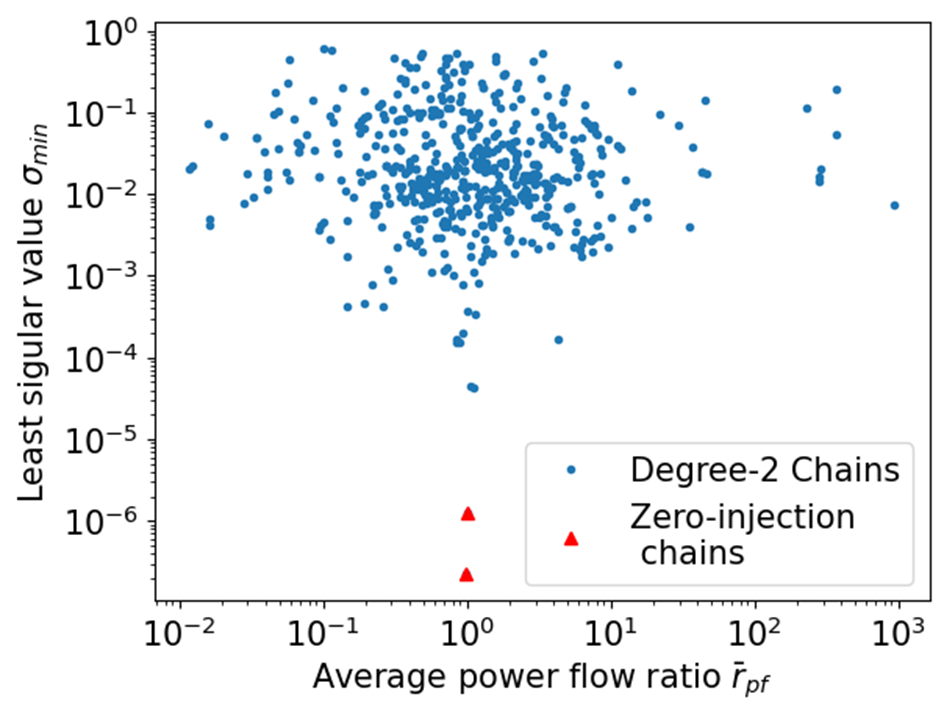}
    \caption{Relationship between the least singular value and the average power-flow ratio for all degree-2 chains in the IEEE 118-bus system under the \ac{ac} model.}
    \label{fig:chain_singular}
\end{figure}

This adjustment allows the method to effectively separate the underlying structure of the measurement data from the perturbed subspace induced by the attack. 
\rev{These results empirically support the geometric foundation of \ac{ping}, which extends effectively to the \ac{ac} model,} maintaining high detection accuracy and low false positive rates despite manifold curvature and the absence of exact linear structure, and outperforming both classical and advanced \ac{ml}-based detectors.

\rev{
\subsection{Threshold Stability Analysis}
\label{sec:threshold-stability}
The threshold $\mathcal{T}$ in \cref{eq:cross_entropy_threshold} is obtained from the empirical histogram of $\|\N_{s,i}\|_2$ via minimum cross-entropy and depends on the noise level $\sigma$ and window length $\time$. 
To assess real-time deployability, we run Monte Carlo sweeps ($N_{MC}=500$) on the IEEE 14- and 118-bus systems.
The trajectory $\Z$ is noised with i.i.d.\ $\mathcal{N}(\mathbf{0},\sigma^2 \mathbf{I})$ draws.
We perform 2 experiments: 1) we vary $\sigma\in[10^{-3},10^{-1}]$~p.u.\ with fixed $\time=1440$ (\Cref{fig:threshold-stability}(a)), and 2) we vary $ \time\in[60,1440]$ with fixed $\sigma=10^{-2}$~p.u (\Cref{fig:threshold-stability}(b)), where each window of length $t$ uses proportional sampling (proportion of attacked samples preserved).
\\
\Cref{fig:threshold-stability}(a) is consistent with the structure of $\|\N_{s,i}\|_2$:
$\mathcal{T}$ saturates at a system-dependent level governed by the bias term $-\U_a\U_a^{+}\sum\a/\time$ when $\sigma\lesssim10^{-2}$, and scales linearly with $\sigma$ above this crossover. 
Notably, both systems share this scaling law, with a relative spread of $<5\%$ across all $\sigma$.
The saturation level for the 118-bus system is approximately five times lower, as the larger null-space dimension averages out per-channel noise.
Fig.~\ref{fig:threshold-stability}(b) reports the \ac{cv}, defined as $\mathrm{CV}(\mathcal{T})=\mathrm{std}(\mathcal{T})/\mathbb{E}[\mathcal{T}]$.
In particular, at $\time=1440$, $\mathrm{CV}(\mathcal{T})\approx2\%$ (IEEE 14) and $<1\%$ (IEEE 118), while at $t=240$, $\mathrm{CV}(\mathcal{T})<7\%$ for both systems. 
The threshold is thus well behaved across $(\sigma,\time)$, and a window size estimate with $\time\geq240$ supports real time deployment.
\begin{figure}[t]
\centering
\includegraphics[width=0.7\linewidth]{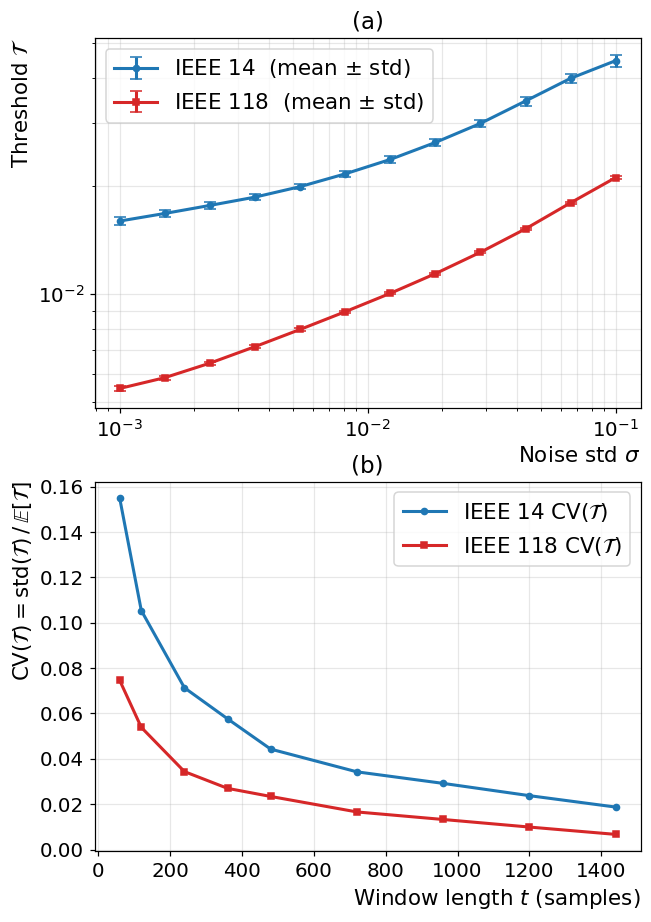}
\caption{Stability of $\mathcal{T}$ on IEEE 14 and 118 ($N_{MC}=500$). (a) $\mathcal{T}$ vs.\ $\sigma$ at fixed
$t=1440$ (error bars: $\pm 1$ std).
(b) $\mathrm{CV}(\mathcal{T})$ vs.\ $t$ at fixed
$\sigma=10^{-2}$~p.u.; with IEEE 118 more constrained due to its higher-dimensional null space.}
\label{fig:threshold-stability}
\end{figure}
}
\section{Conclusion}
\label{sec:conclusion}

This paper presented the \acf{ping} framework for detecting low-magnitude \acfp{fdia} that remain hidden within measurement noise yet distort state estimation. 
\rev{By preserving the physical null space structure using a preprocessing step formalized via null space 
conservation, \ac{ping} identifies small but physically inconsistent perturbations that are overlooked by 
conventional detectors.}

Such low-magnitude attacks can induce significant physical consequences such as a 5\% bias in state estimation as demonstrated in \Cref{sec:results:fdia}. 
Such bias may yield over 10\% voltage-sensitive load mismatch, degrading reactive power balance, frequency regulation~\cite{frequency_control_attack}, and economic dispatch~\cite{xie2011integrity}.  
These findings highlight the need for a detection mechanism that operates without assuming fixed or fully known measurement-space models.
The proposed \acf{pscp} ensures normalization in the physical coordinate frame, improving robustness against heterogeneous \ac{pmu} scaling and numerical drift. The numerical results demonstrate that enforcing physical data alignment offers an effective and scalable defense against stealthy \acp{fdia} in modern power grids. 
As a result, the proposed approach creates a conceptual shift from detecting statistical deviation to identifying physical misalignment, thus linking data-driven analysis with the underlying system physics in a unified manner.

\bibliographystyle{IEEEtran}
\bibliography{mybib, mybib2}

\appendices
\section{Abbreviations}
\begin{acronym}[PCNSA]
\acro{ac}[AC]{alternating current}
\acro{ai}[AI]{artificial intelligence}
\acro{bdd}[BDD]{bad data detector} 
\acro{cnn}[CNN]{convolutional neural network}
\acro{cv}[CV]{coefficient of variation}
\acro{cps}[CPS]{cyber-physical system} 
\acro{dc}[DC]{direct current}
\acro{dl}[DL]{deep learning}
\acro{edr}[EDR]{enforced dynamic response} 
\acro{fdia}[FDIA]{false data injection attack}
\acro{fft}[FFT]{fast Fourier transform}
\acro{gft}[GFT]{graph Fourier transform}
\acro{gnn}[GNN]{graph neural network}
\acro{gsp}[GSP]{graph signal processing}
\acro{imf}[IMF]{intrinsic mode function}
\acro{kcl}[KCL]{Kirchhoff’s current law}
\acro{knn}[KNN]{k-nearest neighbors}
\acro{laa}[LAA]{load altering attack}
\acro{cst}[CST]{Chi-squared test}
\acro{lstm}[LSTM]{long short-term memory}
\acro{ml}[ML]{machine learning}
\acro{mppt}[MPPT]{maximum power point tracking}
\acro{opf}[OPF]{optimal power flow}
\acro{pca}[PCA]{principal component analysis}
\acro{pcc}[PCC]{point of common coupling}
\acro{ping}[PCNSA]{Physically Consistent Null Space Alignment} 
\acro{plc}[PLC]{programmable logic controller}
\acro{pmu}[PMU]{phasor measurement unit}
\acro{pscp}[PSCP]{Pseudo-null Space Conserved data Preprocessing}
\acro{rom}[ROM]{reduced order modeling}
\acro{std}[std.]{standard deviation}
\acrodefplural{std}[std]{standard deviations}
\acro{se}[SE]{state-estimation}
\acro{svd}[SVD]{singular value decomposition}
\acro{vmd}[VMD]{variational mode decomposition}
\acro{wls}[WLS]{weighted least squares}
\acro{wse}[WSE]{wavelet singular entropy} 
\acro{xtm}[XTM]{transformer-based detector}  
\acro{ae}[AE]{autoencoder-based detector}  
\end{acronym}

\end{document}